\documentclass[10pt]{article} % For LaTeX2e
\usepackage[preprint]{tmlr}
\usepackage{amsmath,amsfonts,bm}

\def\eqref#1{equation~\ref{#1}}
\def\1{\bm{1}}

\DeclareMathAlphabet{\mathsfit}{\encodingdefault}{\sfdefault}{m}{sl}
\SetMathAlphabet{\mathsfit}{bold}{\encodingdefault}{\sfdefault}{bx}{n}

\DeclareMathOperator*{\argmax}{arg\,max}

\usepackage[utf8]{inputenc}
\usepackage[T1]{fontenc}    
\usepackage{hyperref}       
\usepackage{url}          
\usepackage{booktabs}     
\usepackage{amsfonts}      
\usepackage{nicefrac}       
\usepackage{microtype}      
\usepackage{xcolor}         
\usepackage{microtype}
\usepackage{graphicx}
\usepackage{subcaption} 
\usepackage{float}

\usepackage{hyperref}

\usepackage{amsmath}
\usepackage{amssymb}
\usepackage{mathtools}
\usepackage{amsthm}
\usepackage{algorithm}
 
\usepackage{algorithmic}
\usepackage{titletoc}

\theoremstyle{plain}
\newtheorem{theorem}{Theorem}[section]

\newtheorem{lemma}[theorem]{Lemma}

\theoremstyle{definition}

\theoremstyle{definition}
\newtheorem{remark}[theorem]{Remark}

\title{Flow Duality and Source Geometry for Categorical Generation}

\author{Etrit Haxholli \\
\texttt{etrit.ks@gmail.com} \\
}

\begin{document}

\maketitle

\begin{abstract}

Continuous and discrete flow matching are usually treated as separate constructions. This paper identifies a duality between them: projecting continuous convex-interpolant flows with one-hot targets through a position-wise argmax yields discrete convex-interpolant flows. The result requires source laws with appropriate coordinate symmetry and regularity, and it makes the continuous source distribution an explicit design choice for categorical generation. We derive the induced discrete interpolation behavior for Gaussian, bounded-uniform, and centered negative-exponential sources, showing that different source geometries lead to qualitatively different transition timing and vocabulary-size dependence. Small visual diagnostics and a short language-modeling pilot suggest that these source-design effects can also appear in learned transports and early generative quality.
\end{abstract}

\section{Introduction}

Generative modeling over discrete objects such as language, code, and symbolic sequences is commonly formulated through discrete diffusion or discrete flow matching. These models define a time-indexed process on a finite state space and learn probability velocities or denoising distributions that move mass from a simple source distribution to the data distribution. 

Recent work on diffusion duality \citep{sahoo2025the} has shown that uniform discrete diffusion processes \citep{lou2023discrete, campbell2022continuous} can be understood as projections of an underlying continuous diffusion process \citep{song2021scorebased}. This raises a natural question for flow matching \citep{lipman2023flow}: does an analogous duality hold for continuous flows with convex interpolants? In particular, if a continuous flow is defined over one-hot token representations \citep{lee2026flowmaplanguage} and then projected back to tokens through the argmax map, what discrete flow is induced?

The main result presented in this paper shows that, under an argmax-compatible lifted coupling and product coordinate-permutation-invariant continuous sources, this projected process is a discrete flow with convex interpolants and uniform source. When the per-position source laws are identical, the induced interpolation coefficient is shared across positions; otherwise the same statement holds with position-dependent coefficients.

This duality makes the continuous source distribution a mathematically meaningful design choice. In one-hot categorical generation, the target coordinate must compete under argmax against the source coordinates. For a Gaussian source, this competition becomes more difficult as the vocabulary size grows, because the target coordinate must overtake the maximum of many noise coordinates. The induced discrete coefficient therefore becomes delayed at large vocabulary size \citep{dieleman2022continuous, sahoo2025the, lee2026flowmaplanguage}. Alternative sources can change this coefficient and, potentially, the geometry of the transport.

We enrich the duality result by deriving the induced discrete coefficient for three source families: Gaussian, bounded uniform, and centered negative-exponential sources. The Gaussian coefficient is vocabulary-dependent and delayed as reported in previous work \citep{dieleman2022continuous, sahoo2025the, lee2026flowmaplanguage}. For fixed vocabulary size, the bounded uniform coefficient is controlled by the support width. The centered negative-exponential coefficient has a simple form independent of vocabulary size. 

We then test the performance of teacher-free two-time denoiser models \citep{lee2026flowmaplanguage} when different sources are used. The results show performance improvements when the aforementioned alternative sources are used. These results are preliminary as the number of training steps is small and the experiments were performed only once.

The paper is organized as follows. Section~2 introduces background and notation. Section~3 presents the duality theorem, its supporting lemmas, source-coefficient extensions, and the preliminary experimental results. All duality proofs and coefficient derivations are given in the appendix.

\section{Background and Notation}
A summary of Discrete Flow Matching is provided below.

\subsection{Discrete Flow Matching}

To expand the design space of discrete diffusion models, \citet{campbell2024generative, gat2024discrete} introduce discrete flow matching. We follow the approach and notation of \citet{gat2024discrete}. In discrete sequence modeling, a sequence (state) $x$ consists of $L$ elements $(x^1, x^2, \dots, x^L)$. Each position $i$ contains an element $x^i$ from a vocabulary $\mathcal{V}$ with $|\mathcal{V}|=V$, which we identify with $[V]=\{1,\dots,V\}$. Thus, the set of possible sequences is $\mathcal{D}=[V]^L$. Two sequences are neighbors if they differ in only one position.

We denote by $p^i(x^i)$ the marginal of $p$ at position $i$, i.e.,
$$p^i(x^i) = \sum_{x^{-i}} p(x)\text{, where } x^{-i} = (x^1,\dots, x^{i-1}, x^{i+1}, \dots, x^L).$$ The following delta function notation will be particularly useful,
\begin{equation}
    \delta_y(x) = \prod_{i=1}^L \delta_{y^i}(x^i), \text{ where } \delta_{y^i}(x^i) = 
\begin{cases}
    1 & \text{if } x^i = y^i \\
    0 & \text{if } x^i \neq y^i
\end{cases}.
\end{equation}

\subsubsection{Probability Flows and Velocities}
In discrete flow matching \citep{gat2024discrete}, the goal is to define or learn a flow $p_t(x):[0,1]\times [V]^L\rightarrow[0,1]$ constrained by $\sum_{x\in [V]^L}p_t(x)=1$ that transforms source (reference) distributions $X_0 \sim p$ to target (data) distributions $X_1 \sim q$. The flow is completely defined by the choice of a probability velocity $u_t(x):[0,1)\times\mathcal{D}\rightarrow\mathbb{R}^{L\times V}$, with $u_t(x)=(u_t^1(x),\dots,u_t^L(x))$ and $u_t^i(x)\in\mathbb{R}^V$. For a candidate token $a\in[V]$, write $u_t^i(a,x):=u_t^i(x)[a]$. These rates satisfy $u_t^i(a,x)\ge 0$ for $a\neq x^i$ and $\sum_{a\in[V]}u_t^i(a,x)=0$. The infinitesimal, first-order update of the probability over states when going from time $t$ to $t+\epsilon$ is defined independently for each position as $p^i_{t+\epsilon|t}(a\mid x)=\delta_{x^i}(a)+\epsilon u_t^i(a,x)+o(\epsilon)$. Therefore, we can see that as in the framework of discrete diffusion, the probability over the states in the next step depends solely on the current state, and that $u_t$ plays a similar role to a transition-rate matrix ${Q}_t$, completely determining the flow. As such, if we approximate the probability velocity $u_t(x)$ using a neural network $u_t(x,t;\theta): [V]^L\times[0,1]\rightarrow \mathbb{R}^{L\times V}$, we can sample from $p$ and generate data from $q$, using the previous update rule. Before modeling the probability velocity $u_t(x)$ however, one must first design an appropriate flow $p_t(x)$ that has a suitable, practically learnable corresponding $u_t(x)$.

\subsubsection{Conditional Probability Flows}\label{cpf}

 Since at time $t=0$ and $t=1$ we must have $p_0 = p$ and $p_1 = q$ respectively, we are already restricted regarding the endpoints of the flow. A trivial way to satisfy such constraints is to define 
\begin{equation}\label{main_flow_def}
    p_t(x) = \sum_{x_0, x_1 \in \mathcal{D}} p_t(x|x_0, x_1)\pi(x_0, x_1),
\end{equation}
where $p_0(x|x_0, x_1)=\delta_{x_0}(x)$, $p_1(x|x_0, x_1)=\delta_{x_1}(x)$ and $ \pi(X_0, X_1)$ is an arbitrary joint distribution of $X_0$, $X_1$ satisfying the marginals constraints $p(x) = \sum_{y \in \mathcal{D}} \pi(x, y)$, $q(y) = \sum_{x \in \mathcal{D}} \pi(x, y)$. Since the probability velocities update the probability independently for each position, it is natural to define $p_t(x|x_0, x_1)$ independently for each dimension as in \citet{gat2024discrete}:
\begin{equation}\label{pos_independ}
p_t(x|x_0, x_1) = \prod_{i=1}^L p^i_t(x^i|x_0, x_1), 
\end{equation}
\begin{equation}\label{convex_cond_flow}
\text{where } p_t^i(x^i | x_0, x_1) =(1 - k_t) \delta_{x_0^i}(x^i) + k_t \delta_{x_1^i}(x^i),  \text{with } k_0=0, k_1=1 \text{ and increasing } k_t.\ \ \ \ \ \ \ \ \ \ \ \ \ \ \ \ \ \ \ \ \ \ 
\end{equation}
It is clear that this definition of $p_t(x|x_0, x_1)$ satisfies the conditions $p_0(x|x_0, x_1)=\delta_{x_0}(x)$ and $p_1(x|x_0, x_1)=\delta_{x_1}(x)$. Assume additionally that $k_t$ is absolutely continuous on $[0,1]$. Then \citet{gat2024discrete} show that component $i$ of the conditional probability velocity $u_t^i(a,x\mid x_0,x_1)$ corresponding to the flow defined in Equations (\ref{pos_independ}) and (\ref{convex_cond_flow}) is
\begin{equation}\label{convex_cond_velocity}
u_t^i(a, x\mid x_0, x_1) =\frac{\dot{k}_t}{1 - k_t}  \left[\delta_{x_1^i}(a)-\delta_{x^i}(a)\right].
\end{equation}
Furthermore, they show that the probability velocity corresponding to the unconditional flow $p_t(x)$ can be written as 
\begin{equation}\label{uncond_vect_field}
u_t^i(a, x) =  \sum_{x_0, x_1 \in \mathcal{D}}
u_t^i(a, x\mid x_0, x_1)\, p_{0,1|t}(x_0, x_1\mid x),
\end{equation}
 which in the case of Equations (\ref{convex_cond_flow}) and (\ref{convex_cond_velocity}) implies, $u_t^i(a, x)= \frac{\dot{k}_t}{1 - k_t} \left[ p_{1|t}^i(a \mid x,t) - \delta_{x^i}(a) \right].$
One then approximates $u_t^i(a, x)$ by modeling the time-conditioned posterior predictor with a neural network $p^i_{1|t}(a\mid x,t;\theta)$  using the cross-entropy loss, 
\begin{equation}\label{optimization_objective}
\mathcal{L}_{\mathrm{DFM}}
= -\mathbb{E}_{t \sim U(0,1)} \mathbb{E}_{x_0, x_1\sim \pi(x_0, x_1)}\mathbb{E}_{x_t \sim p_{t|0,1}(\cdot |x_0, x_1)} \sum_{i=1}^L \log p^i_{1|t}(x_1^i \mid x_t,t;\theta).
\end{equation}
It should be mentioned that in \citet{gat2024discrete}, the definition of $p_t^i(x^i | x_0, x_1)$ is given in a more general form, but here we focus on this specific case for simplicity.

\subsubsection{Source and Target Distributions}
As mentioned, points $X_0$ and $X_1$ are sampled from a joint distribution $\pi(x, y)$, i.e. $(X_0, X_1) \sim \pi(X_0, X_1)$, satisfying the marginals constraints $p(x) = \sum_{y \in \mathcal{D}} \pi(x, y)$, $q(y) = \sum_{x \in \mathcal{D}} \pi(x, y)$. 
As a special case, the training pairs $X_0$ and $X_1$ can be sampled independently, $(X_0, X_1) \sim p(X_0) q(X_1)$.
Common instantiations of source distribution $p$ are: 
\newline
(i) adding a special token value often referred to as a \emph{mask} token, denoted here by $m$. This uses an enlarged vocabulary $\mathcal V_{\mathrm{mask}}=[V]\cup\{m\}$, outside the strict $[V]^L$ setup unless $m$ is already included in the vocabulary. The source is $X_0=(m,\dots,m)$ almost surely, with $X_1\sim q$ under the chosen coupling.
\newline
(ii) using uniform distribution over $\mathcal{D}$, which is equivalent to drawing each $x^i$ independently to be some value in $[V]$ with equal probability, denoted $p_u(x^i)$.

\subsubsection{Continuous Flows}\label{cont_flow_sec}

Let $\tilde{\pi}$ be a probability measure on $\mathbb{R}^{LV}\times\mathbb{R}^{LV}$, with first and second marginals $\tilde p_0$ and $\tilde p_1$. We refer to $\tilde p_0$ as the continuous source and $\tilde p_1$ as the continuous target. This measure notation is important in the present setting because the target may be supported on one-hot vectors and therefore need not admit a Lebesgue density.

For a nondecreasing interpolant coefficient $\tilde{k}_t:[0,1]\to[0,1]$ satisfying $\tilde{k}_0=0$ and $\tilde{k}_1=1$, define
\begin{equation}
    T_t(z_0,z_1)=(1-\tilde k_t)z_0+\tilde k_t z_1.
\end{equation}
The continuous flow \citep{lipman2023flow, tong2023improving} is the pushforward
\begin{equation}
    \tilde p_t = (T_t)_\#\tilde\pi.
\end{equation}
Equivalently, conditional on a coupled endpoint pair $(z_0,z_1)$, the trajectory is deterministic and has the form
\begin{equation}
    z_t=(1-\tilde k_t)z_0+\tilde k_t z_1.
\end{equation}
Thus, the conditional trajectories are straight line segments between coupled endpoints, possibly with a time reparameterization through $\tilde k_t$. When a classical continuous velocity field is discussed, we additionally assume $\tilde k_t$ is absolutely continuous, hence differentiable almost everywhere.

\subsubsection{Diffusion duality}\label{subsec:diffusion_duality_background}

\citet{sahoo2025the} showed that a uniform-state discrete diffusion process can be realized as the argmax projection of an underlying Gaussian diffusion. In this view, the continuous latent process evolves in Euclidean space while the observed categorical state is obtained by applying argmax to each vocabulary block. This makes it possible to transfer continuous-diffusion ideas to discrete generation, including lower-variance training curricula and consistency-style distillation. The result motivates the question studied here: whether an analogous projection principle holds for more general continuous flows with convex interpolants rather than diffusion noising processes.

\subsubsection{Two-time denoiser and flow maps}\label{subsec:two_time_denoiser_background}

\citet{lee2026flowmaplanguage} study continuous flows over one-hot token embeddings at the target distribution. Let $\Phi_{s,t}$ denote the map that transports a continuous state from time $s$ to time $t$, and define the average velocity
\begin{equation}
    v_{s,t}(x) := \frac{\Phi_{s,t}(x)-x}{t-s}.
\end{equation}
Their two-time denoiser is the reparameterization
\begin{equation}
    \delta_{s,t}(x) := x+(1-s)v_{s,t}(x),
\end{equation}
which can be inverted to recover the flow map:
\begin{equation}
    \Phi_{s,t}(x)
    = \frac{1-t}{1-s}x+\frac{t-s}{1-s}\delta_{s,t}(x).
\end{equation}
For one-hot targets, $\delta_{s,t}$ lies on the probability simplex, so it can be represented by a token-wise softmax and trained with cross-entropy and KL objectives (e.g. the semi-group condition). At $s=t$ it reduces to the usual endpoint denoiser, while at $t=1$ it gives the final clean prediction map.

\subsubsection{Notation summary}

For reference, the main probability symbols used below are summarized here:
\begin{center}
\begin{tabular}{ll}
\toprule
Symbol & Meaning \\
\midrule
$\tilde p_t$ & continuous marginal law on $\mathbb R^{LV}$ \\
$p_t=(\argmax)_{\#}\tilde p_t$ & discrete argmax pushforward law on $[V]^L$ \\
$p_{t|0,1}$ & discrete conditional law of $X_t$ given $(X_0,X_1)$ \\
$p_{t|0,1}^i$ & per-position conditional law at position $i$ \\
$q$ & target/data law on $[V]^L$ \\
$\pi(x_1\mid x_0)$ & discrete coupling kernel used in the lifted construction \\
$\bar\pi(x_0,x_1)$ & induced discrete joint $p_0(x_0)\pi(x_1\mid x_0)$ \\
$\Phi_{s,t}$ & continuous flow map from time $s$ to time $t$ \\
$v_{s,t}$ & average velocity associated with $\Phi_{s,t}$ \\
$\delta_{s,t}$ & two-time denoiser of \citet{lee2026flowmaplanguage} \\
$D_t$ & endpoint denoiser, recovered as the diagonal case of $\delta_{s,t}$ \\
\bottomrule
\end{tabular}
\end{center}

\section{Duality}
In what follows, we demonstrate that applying the $argmax$ operator position-wise along a continuous flow with convex interpolants results in a corresponding discrete convex-interpolant flow. Moreover, if the continuous source is coordinate-permutation invariant, then the corresponding discrete source distribution is uniform. This establishes a connection between discrete convex-interpolant flows with uniform sources and continuous convex-interpolant flows with coordinate-permutation invariant source distributions. Our result generalizes the diffusion-duality viewpoint of \citet{sahoo2025the} to the setting of flows. All proofs can be found in Appendix \ref{A}.

\subsection{The Duality Between Discrete and Continuous Flows with Convex Interpolants}

 In what follows, $\tilde p_t$ denotes the continuous marginal law at time $t$; when this law admits a density, we use the same symbol for that density by abuse of notation.
 
We define a continuous flow as in Section \ref{cont_flow_sec}, where each block $z^i \in \mathbb{R}^V$ corresponds to the $i$-th position:
\begin{equation}
z^i := z[(i-1)V+1 : iV] \in \mathbb{R}^V,
\qquad i=1,\dots,L.
\end{equation}

We define the per-position mapping
\begin{equation}
\argmax : \mathbb{R}^V \to [V]
\end{equation}
as follows. First, for $u \in \mathbb{R}^V$, let
\begin{equation}
S(u) := \{j \in [V] : u_j = \max_{k \in [V]} u_k\}.
\end{equation}
We then define
\begin{equation}
\argmax(u) := \min S(u).
\end{equation}
Thus ties are broken by taking the smallest index. We extend this to sequences by
\begin{equation}
\argmax : \mathbb{R}^{LV} \to [V]^L, 
\qquad 
(\argmax(z))^i := \argmax(z^i),\ i=1,\dots,L.
\end{equation}

Finally, we denote by $\mathcal{B}(\cdot)$ the Borel $\sigma$-algebra, and by $\mathcal{P}([V]^L)$ the power set of $[V]^L$.

Now, we first verify that $\argmax$ is measurable and thus induces a pushforward of any probability measure on $\mathbb{R}^{LV}$ to a probability measure on $[V]^L$.

\begin{lemma}
\label{lem:argmax-measurable}
The mapping $\argmax : (\mathbb{R}^{LV}, \mathcal{B}(\mathbb{R}^{LV})) \to ([V]^L, \mathcal{P}([V]^L))$ is measurable.
\end{lemma}

As a consequence, for any probability measure $\mu$ on $(\mathbb{R}^{LV}, \mathcal{B}(\mathbb{R}^{LV}))$, $(\argmax)_{\#}\mu$ is a well-defined probability measure on $([V]^L, \mathcal{P}([V]^L))$. In particular, for the continuous marginal $\tilde p_t$ we define
\begin{equation}
 p_t := (\argmax)_{\#}\tilde{p}_t.
\end{equation} For time $t=0$ we have the following result:

\begin{lemma}[Uniform discrete source]
\label{lem:uniform-source}
For continuous source distributions $\tilde{p}_0=\prod_{i=1}^L \tilde{p}^i_0$, assume each $\tilde{p}^i_0$ is coordinate-permutation invariant and has no ties almost surely, i.e.
\begin{equation}
\tilde p_0^i\bigl(\{u\in\mathbb R^V:\exists k\neq \ell,\ u_k=u_\ell\}\bigr)=0.
\end{equation}
Then the pushforward $p_0=(\argmax)_{\#}\tilde p_0$ is the uniform distribution on $[V]^L$. In particular, for each position $i$ and each token $j \in [V]$,
\begin{equation}
p_0^i(j) = \frac{1}{V},
\end{equation}
where $x_0 := \argmax(z_0)$ and $z_0 \sim \tilde{p}_0$.
\end{lemma}

For time $t=1$, we have the special case $z_1 = \operatorname{onehot}(x_1)$. So the $\argmax$ map only undoes the $\operatorname{onehot}$ encoding $\argmax (z_1) = \argmax (\operatorname{onehot}(x_1))=x_1$.

The factorization results below use the following argmax-compatible lifting of a discrete coupling. First sample $Z_0\sim \tilde p_0$ and set $X_0=\argmax(Z_0)$. Then sample $X_1$ from a discrete coupling kernel $\pi(\cdot\mid X_0)$ using randomness conditionally independent of $Z_0$ given $X_0$. When the endpoint is intended to be the data distribution $q$, we assume this kernel has marginal $X_1\sim q$, i.e.
\begin{equation}
q(x_1)=\sum_{x_0\in[V]^L}p_0(x_0)\,\pi(x_1\mid x_0)=\frac{1}{V^L}\sum_{x_0\in[V]^L}\pi(x_1\mid x_0).
\end{equation}
Equivalently, the induced discrete joint is $\bar\pi(x_0,x_1)=p_0(x_0)\pi(x_1\mid x_0)$. Finally set $Z_1=\operatorname{onehot}(X_1)$. Then define the continuous convex interpolant
\begin{equation}
    Z_t=(1-\tilde k_t)Z_0+\tilde k_t Z_1,
    \qquad X_t=\argmax(Z_t).
\end{equation}
This construction allows an arbitrary discrete coupling kernel between $X_0$ and $X_1$ (with the desired endpoint marginal imposed when needed) subject to the argmax-compatible conditional-independence rule above, but it is not an arbitrary continuous coupling $\tilde\pi(z_0,z_1)$: conditional on $X_0$, the endpoint $X_1$ depends on $Z_0$ only through $X_0=\argmax(Z_0)$. The kernel $\pi(x_1\mid x_0)$ may correlate target positions, but once $x_1$ is fixed, the remaining randomness in $X_t$ comes from the conditionally independent source blocks $Z_0^i\mid X_0^i=x_0^i$. This conditional-independence structure is what preserves the per-position factorization used below.

Given that we know how the source and target distributions are transformed, the remaining object is the per-position conditional $p^i_{t|0,1}(x_t^i\mid x_0^i,x_1^i)$. 

The following result is crucial. It is precisely where the conditional independence $Z_0\perp X_1\mid X_0$ is used: without it, conditioning on $X_1$ could introduce dependence among the within-cell source blocks.

\begin{lemma}[Factorization across positions]
\label{lem:factorization}
Under the lifted coupling construction above, for each $t \in [0,1]$ and any pair $(x_0,x_1)$ in the support of the lifted discrete coupling, the conditional distribution of $x_t$ given $(X_0,X_1)=(x_0,x_1)$ factorizes across positions:
\begin{equation}
p_{t|0,1}(x_t \mid x_0,x_1)
= \prod_{i=1}^L p_{t|0,1}^i(x_t^i \mid x_0^i,x_1^i),
\end{equation}
for some per-position conditionals $p_{t|0,1}^i(\cdot \mid x_0^i,x_1^i)$ on $[V]$.
\end{lemma}
The only remaining missing piece is the form of $p_{t|0,1}^i(x_t^i \mid x_0^i,x_1^i)$. This is answered by the following theorem which also summarizes previous findings:

\begin{theorem}[Discrete conditional path induced by a continuous flow]
\label{thm:discrete-flow}
Assume $V\ge 2$ and let $\tilde k:[0,1]\to[0,1]$ be nondecreasing with $\tilde{k}_0=0$ and $\tilde{k}_1=1$. For each $t$, write $c_t:=\tilde k_t/(1-\tilde k_t)$, with the convention $c_t=+\infty$ when $\tilde k_t=1$. Assume that the continuous source factorizes across positions, $\tilde p_0=\prod_{i=1}^L\tilde p_0^i$, that each per-position source law $\tilde p_0^i$ is coordinate-permutation invariant, and that if $Z^i\sim\tilde p_0^i$, then $Z^i$ has no ties almost surely and, for every $a\neq b$, the pairwise gap $Z^i_a-Z^i_b$ has no atom at $c_t$ for each $t$ with $\tilde k_t<1$. Assume also that the coupling is the argmax-compatible lifted coupling described above, in particular satisfying $\mathcal L(X_1\mid Z_0)=\pi(\cdot\mid\argmax Z_0)$: $Z_0\sim\tilde p_0$, $X_0=\argmax(Z_0)$, sample $X_1$ from $\pi(\cdot\mid X_0)$ using randomness conditionally independent of $Z_0$ given $X_0$ (equivalently, $\mathcal L(X_1\mid Z_0)=\pi(\cdot\mid\argmax Z_0)$), and set $Z_1=\operatorname{onehot}(X_1)$. For each $t\in[0,1]$, define $Z_t=(1-\tilde k_t)Z_0+\tilde k_t Z_1$ and $\tilde p_t:=\mathcal L(Z_t)$. Then the law $p_t = (\argmax)_{\#}\tilde{p}_t$ is a probability distribution on $[V]^L$ with the following properties:
\begin{enumerate}
    \item[(i)] The source $p_0$ is uniform on $[V]^L$.
    \item[(ii)] If the lifted discrete kernel has endpoint marginal
\begin{equation}
q(x_1)=\sum_{x_0\in[V]^L}p_0(x_0)\,\pi(x_1\mid x_0)=\frac{1}{V^L}\sum_{x_0\in[V]^L}\pi(x_1\mid x_0),
\end{equation}
then the terminal marginal satisfies $p_1=q$.
    \item[(iii)] Interpolation probabilities factor per position,
    \begin{equation}
        p_{t|0,1}(x_t \mid x_0,x_1)
        = \prod_{i=1}^L p_{t|0,1}^i(x_t^i \mid x_0^i,x_1^i).
    \end{equation}
    \item[(iv)] The unconditional discrete path is the mixture
    \begin{equation}
    p_t(x)=\sum_{x_0,x_1\in[V]^L}p_{t|0,1}(x\mid x_0,x_1)\,\bar\pi(x_0,x_1),
    \qquad
    \bar\pi(x_0,x_1)=p_0(x_0)\pi(x_1\mid x_0).
    \end{equation}
    \item[(v)] For any pair $(x_0,x_1)$ in the support of the lifted discrete coupling and each position $i$, there exists $k_t^i\in[0,1]$ such that
    \begin{equation}
    p_{t|0,1}^i(x_t^i \mid x_0^i,x_1^i)
    = (1-k_t^i)\,\delta_{x_0^i}(x_t^i) + k_t^i\,\delta_{x_1^i}(x_t^i).
    \end{equation}
    For $x_0^i\neq x_1^i$, this coefficient is
    \begin{equation}
    k_t^i
    =
    \mathbb P\!\left(
    Z^i_{x_0^i}-Z^i_{x_1^i}
    \le
    c_t
    \;\middle|\;
    \argmax(Z^i)=x_0^i
    \right),
    \end{equation}
    where $Z^i\sim\tilde p_0^i$. For distinct tokens, $t\mapsto k_t^i$ is nondecreasing, and $k_0^i=0$, $k_1^i=1$ under the conventions $c_0=0$ and $c_1=+\infty$. When $x_0^i=x_1^i$, the two delta masses coincide, so the value of the coefficient is immaterial; for notational consistency we assign it the same value as in the distinct-token case. Thus $k_t^i$ depends only on $\tilde{k}_t$ and the $i$-th continuous source law $\tilde p_0^i$, not on $(x_0,x_1)$. If the per-position source laws are identical, then $k_t^i=k_t$ for all $i$.
\end{enumerate}
In other words, the per-position conditionals of the discrete conditional path generated by the continuous one coincide with the standard convex-interpolant form used in discrete flow matching. In the common-source case, the usual single scalar coefficient $k_t$ is shared across positions.
\end{theorem}

\begin{remark}[Independent endpoints]
\label{rem:independent-endpoints}
When $Z_0\sim\tilde p_0$ and $Z_1=\operatorname{onehot}(X_1)$,
with $X_1\sim q$, are sampled independently, the required conditional
independence holds automatically, since $X_0=\argmax(Z_0)$ is a
deterministic function of $Z_0$:
\begin{equation}
    Z_0\perp\!\!\!\perp Z_1\mid X_0.
\end{equation}
Thus, no discrete coupling kernel needs to be specified separately:
independence induces $\pi(x_1\mid x_0)=q(x_1)$, and the required
endpoint marginal is automatically satisfied.
\end{remark}
\begin{remark}[Mask-source variants]
\label{subsec:reverse_lifts}
A similar result can be stated for multi-mask sources
\citep{haxholli2026minibatch}. One possible construction extends
the vocabulary from $[V]$ to $[2V]$, as well as $\mathbb{R}^{LV}$
to $\mathbb{R}^{2LV}$, and uses independent coordinates with
separated uniform supports, for example $\mathcal{U}(a,b)$ on
data-token coordinates $[1:V]$ and $\mathcal{U}(c,d)$ on mask
coordinates $[V+1:2V]$ with $c>b$.
\end{remark}
\begin{remark}[Min--max target encoding]
\label{rem:min-max-target}
When $G$ is atomless and supported on $(0,\infty)$, we can randomize
the target encoding while preserving its argmax. At each position
$i$, draw two independent samples from $G$. Place the larger sample
in coordinate $X_1^i$ and the smaller sample in a uniformly chosen
different coordinate, leaving the remaining entries equal to zero.
Since the samples are positive and distinct almost surely, the
resulting vector still has argmax $X_1^i$.

Conditional on $(X_0,X_1)$, take all additional draws to be mutually
independent and independent of $Z_0$. Under the same source and
coupling assumptions as above, the discrete dual of the resulting
convex interpolant has the per-position conditional form of
Equation~(10) in \citet{gat2024discrete} whenever $x_0^i\ne x_1^i$.
\end{remark}

\subsection{Source-dependent coefficients induced by the duality}\label{subsec:source_coefficients}

Theorem~\ref{thm:discrete-flow} shows that the projected continuous flow has the same conditional form as a discrete convex-interpolant flow. In the common per-position source setting, the remaining question is how the shared scalar coefficient $k_t$ depends on the continuous source. This coefficient is the probability, conditional on the source token and target token being distinct, that the target coordinate has overtaken the source coordinate by time $t$. Thus the source distribution does not merely choose where the continuous trajectories begin; it determines the effective time dependence of the induced discrete flow.

For the following source families, assume $V\ge 2$ and write
\begin{equation}
    c_t := \frac{\tilde{k}_t}{1-\tilde{k}_t},
\end{equation}
with the convention that $c_t=+\infty$ when $\tilde{k}_t=1$. The derivations are provided in Appendix~\ref{app:coefficient_derivations}.

\paragraph{Gaussian source.}
For i.i.d. standard Gaussian coordinates, the induced coefficient is
\begin{equation}\label{eq:gaussian_coeff_body}
    k_t^{\mathrm{Gauss}}
    = V\int_{-\infty}^{\infty}
    \varphi(m)\Phi(m)^{V-2}\bigl(\Phi(m)-\Phi(m-c_t)\bigr)\,dm,
\end{equation}
where $\varphi$ and $\Phi$ denote the standard normal density and distribution function. Equivalently, if
\begin{equation}
    \rho_t^{\mathrm{Gauss}}
    = \mathbb{E}_{G\sim N(0,1)}\left[\Phi(G+c_t)^{V-1}\right],
\end{equation}
then
\begin{equation}
    k_t^{\mathrm{Gauss}}=\frac{V\rho_t^{\mathrm{Gauss}}-1}{V-1}.
\end{equation}
This form makes the vocabulary dependence explicit: as $V$ grows, the target coordinate has to compete with the maximum of many Gaussian coordinates, delaying the induced transition.

\paragraph{Bounded uniform source.}
For i.i.d. uniform coordinates on an interval of width $\epsilon>0$, the induced coefficient is
\begin{equation}\label{eq:uniform_coeff_body}
    k_t^{\mathrm{Unif}}
    = \frac{Vr_t-r_t^V}{V-1},
    \qquad
    r_t = \operatorname{clip}\!\left(\frac{c_t}{\epsilon},0,1\right).
\end{equation}
Here $\operatorname{clip}(a,0,1)=\min\{\max\{a,0\},1\}$. For fixed vocabulary size $V$, the coefficient depends on the width of the interval, not on its location. Shifting the interval affects the geometry of the continuous trajectories, while the induced discrete coefficient is controlled by the support width and $V$. Because of the clipping, this coefficient can reach $1$ before the endpoint: under the linear schedule $\tilde k_t=t$, we have $c_t=t/(1-t)$, so $r_t=1$ once $t\ge \epsilon/(1+\epsilon)$. After that time the induced conditional path has already reached the target, and any velocity formula containing $(1-k_t^{\mathrm{Unif}})^{-1}$ should be interpreted only on the interval where $k_t^{\mathrm{Unif}}<1$.

\paragraph{Centered negative-exponential source.}
Let $\beta>0$ and $Z=\beta^{-1}-E$ where $E\sim\mathrm{Exp}(\beta)$; the shift makes $\mathbb E[Z]=0$. For i.i.d. coordinates of this form, the induced coefficient is
\begin{equation}\label{eq:negexp_coeff_body}
    k_t^{\mathrm{NegExp}} = 1-\exp(-\beta c_t).
\end{equation}
Unlike the Gaussian coefficient, this expression has no explicit vocabulary-size dependence.

Figures~\ref{fig:coefficients_time} and~\ref{fig:gaussian_scaling} visualize these coefficients under the linear interpolant $\tilde{k}_t=t$. Figure~\ref{fig:coefficients_time} compares the three source families at GPT-2 vocabulary scale, while Figure~\ref{fig:gaussian_scaling} isolates the vocabulary dependence of the Gaussian source.

\begin{remark}[Random-rate negative-exponential source]
\label{rem:random_rate_negexp}
A hierarchical variant of the centered negative-exponential source recovers the continuous interpolation schedule exactly. For each position $i$, sample $A_i\sim\operatorname{Exp}(1)$ and, conditional on $A_i=a$, let the $V$ coordinates be i.i.d.\ as $a^{-1}-E$ with $E\sim\operatorname{Exp}(a)$. Since the source argmax is independent of $A_i$, averaging the fixed-rate coefficient $1-e^{-a c_t}$ over $A_i$ gives

$$
k_t
=
1-\mathbb{E}[e^{-A_i c_t}]
=
1-\frac{1}{1+c_t}
=
\tilde{k}_t,
\qquad
c_t=\frac{\tilde{k}_t}{1-\tilde{k}_t}.
$$

Thus the induced discrete coefficient coincides with the continuous one and is independent of the vocabulary size $V$. However, this source lacks a finite first absolute moment, which may limit its practical usefulness.
\end{remark}

\begin{figure}[H]
    \centering
    \includegraphics[width=0.75\linewidth]{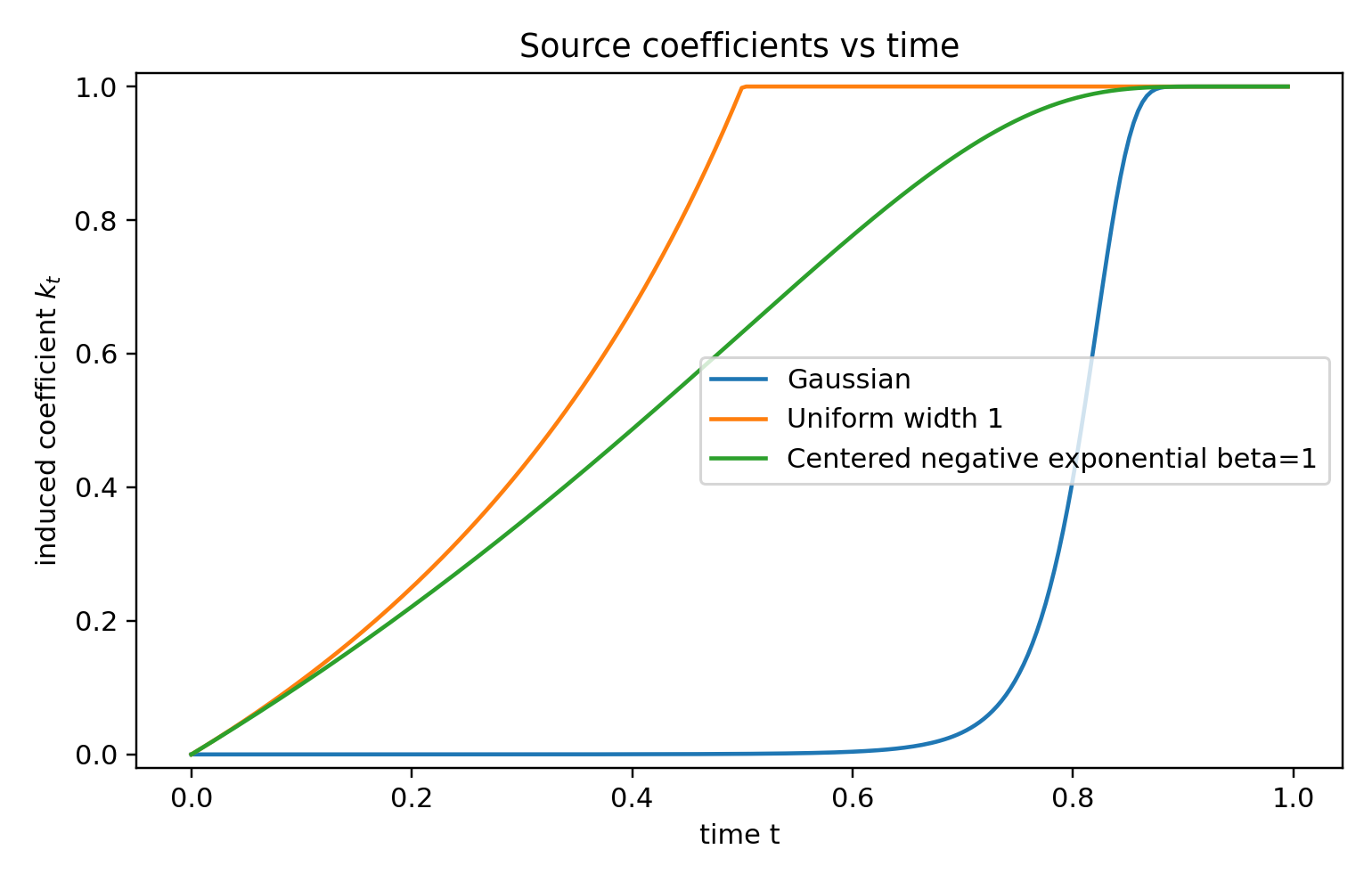}
    \caption{Induced coefficient $k_t$ for Gaussian, bounded uniform with $\epsilon=1$, and centered negative-exponential with $\beta=1$ under the linear interpolant $\tilde{k}_t=t$ and vocabulary size $V=50{,}257$.}
    \label{fig:coefficients_time}
\end{figure}

\begin{figure}[H]
    \centering
    \includegraphics[width=0.75\linewidth]{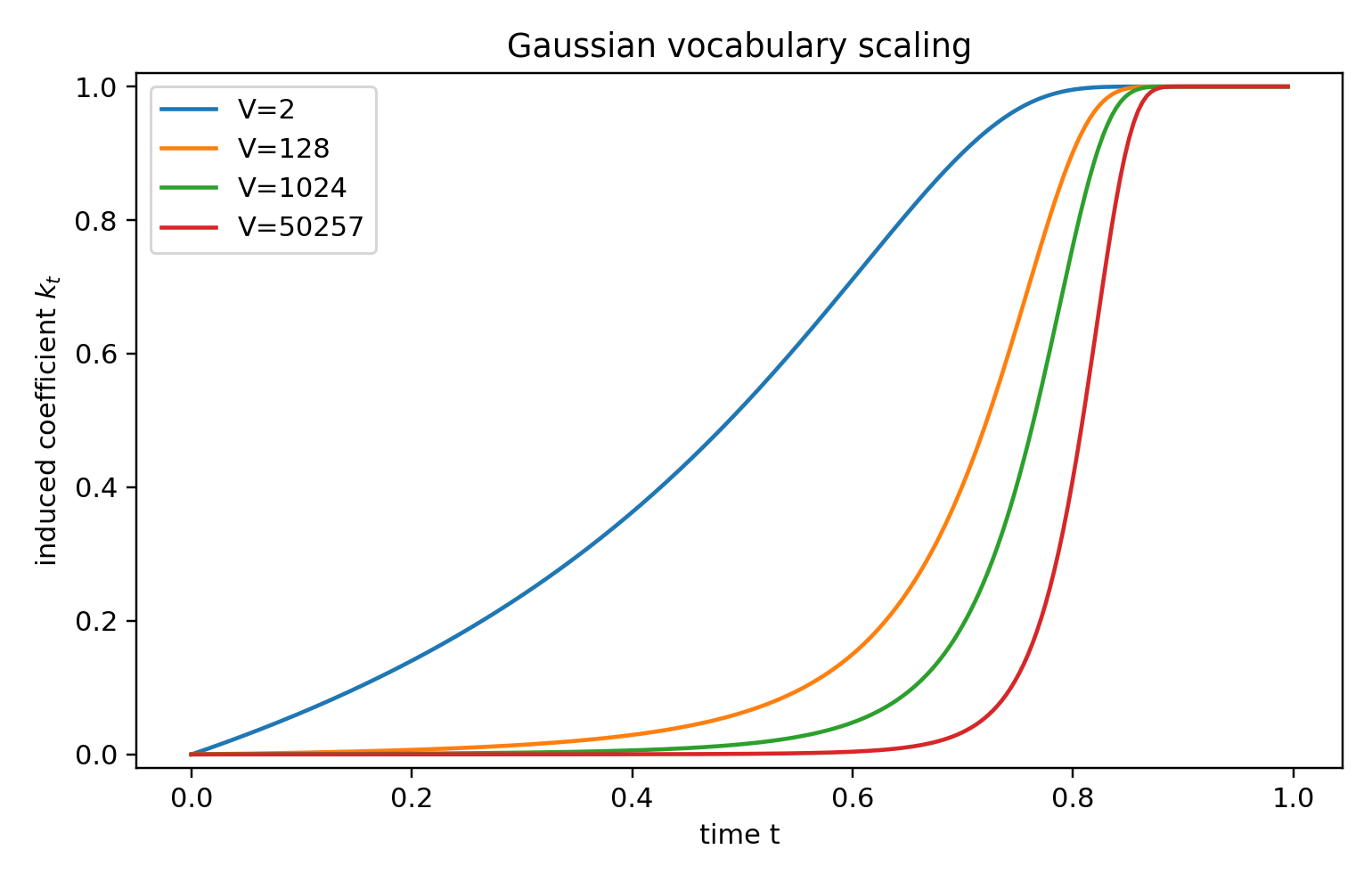}
    \caption{Vocabulary-size dependence of the Gaussian induced coefficient under the linear interpolant $\tilde{k}_t=t$. Larger vocabularies shift the effective transition later in time.}
    \label{fig:gaussian_scaling}
\end{figure}

\subsection{Tiny 10k-step generative check}\label{subsec:tiny_experiment}

As a small sanity check, we ran short 10k-step OpenWebText pilot trainings comparing the Gaussian source to a centered negative-exponential source and a uniform source. In all cases we trained a two time denoising model \citet{lee2026flowmaplanguage} with identical architecture and hyperparameters. For the Gaussian source, we use the decoding-error time reparameterization of \citet{lee2026flowmaplanguage}, which they show substantially reduces generative perplexity compared to the non-reparameterized case while leaving sample entropy essentially unchanged. For the bounded-uniform and centered negative-exponential sources, we instead feed the continuous time parameter directly (non-reparameterized) into the network.

Following the evaluation decomposition advocated by \citet{pynadath2026generative}, we combine generative perplexity and sample entropy into the plug-in estimate
\begin{equation}
\widehat D_{\mathrm{KL}}
=
\log(\mathrm{GenPPL})-\widehat H,
\end{equation}
where $\log(\mathrm{GenPPL})$ estimates the cross-entropy against the reference language model and $\widehat H$ is the reported unigram entropy, used as a surrogate for joint sequence entropy per token as in \citet{pynadath2026generative}. This estimate should be read only as a single operating-point diagnostic, not as a replacement for a full generative frontier. Samples were generated in a single step.

\begin{table}[H]
\centering
\caption{Tiny 10k-step OpenWebText pilot. Lower GenPPL, cross-entropy, and estimated KL are better; higher sample entropy is better. One step generation. Single training run.}
\label{tab:tiny_10k_eval}
\begin{tabular}{lrrrr}
\toprule
Source & GenPPL $\downarrow$ & $\log(\mathrm{GenPPL})$ $\downarrow$ & Entropy $\uparrow$ & $\widehat D_{\mathrm{KL}}$ $\downarrow$ \\
\midrule
Gaussian & 428.88 & 6.061 & 4.148 & 1.913 \\
Centered neg.-exp. & 258.51 & 5.555 & 4.022 & 1.533 \\
$U(-4,-3)$ & 244.68 & 5.497 & 4.05 & 1.447 \\
\bottomrule
\end{tabular}
\end{table}

In this early-training run, the centered negative-exponential source reduces estimated KL by 0.380 nats relative to the Gaussian source, a 19.9\% relative reduction, while sample entropy decreases only modestly. The result is preliminary, but it is consistent with the source-geometry prediction that changing the continuous source can alter generative behavior before long training.

\begin{figure}[H]
    \centering
    \begin{subfigure}[t]{0.49\linewidth}
        \centering
        \includegraphics[width=\linewidth]{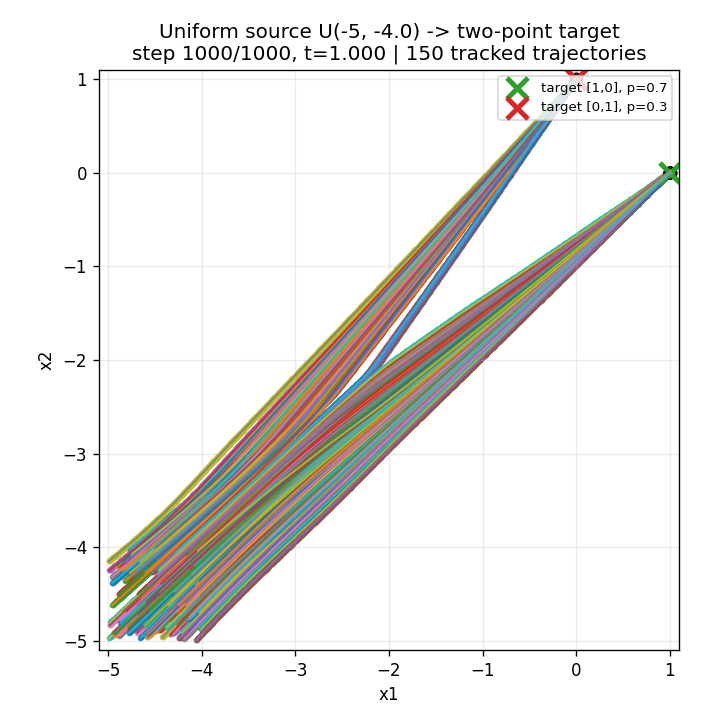}
        \caption{Shifted bounded uniform source.}
    \end{subfigure}
    \hfill
    \begin{subfigure}[t]{0.49\linewidth}
        \centering
        \includegraphics[width=\linewidth]{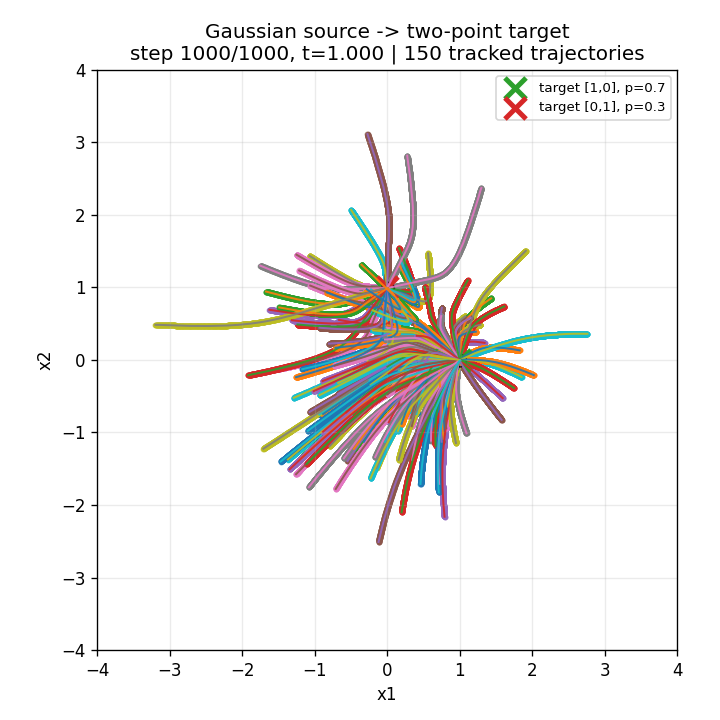}
        \caption{Gaussian source.}
    \end{subfigure}
    \caption{Toy tracked trajectories for a learned flow in a two-point target problem. The theorem's conditional convex interpolants are straight line segments; these plots instead visualize trajectories of a learned marginal toy transport tracked over a fixed integration grid.}
    \label{fig:toy_trajectories}
\end{figure}

\subsection{Toy trajectory geometry}\label{subsec:trajectory_geometry}

The coefficient $k_t$ describes the induced discrete path associated with the conditional convex interpolants. Figure~\ref{fig:toy_trajectories} instead shows trajectories produced by a learned marginal toy flow in a two-point target experiment with vocabulary size two and sequence length one. The same small neural vector-field model, loss, and integration grid are used for both sources. The plot is purely diagnostic: it visualizes learned transport trajectories, not the theorem's conditional straight-line interpolants. The comparison also confounds source family with source placement and scale relative to the one-hot targets, so it should be read only as a qualitative illustration of possible geometric effects.

\section{Conclusion}

The central result of this paper is that continuous flows with convex interpolants, one-hot encoded targets, and product coordinate-permutation-invariant continuous sources induce discrete conditional flows with convex-interpolant conditionals after position-wise argmax projection. The supporting lemmas show that the argmax map is measurable, that symmetric continuous sources push forward to uniform discrete sources, and that the projected conditional law factors across positions, with each factor being the convex-interpolation in the discrete flow matching framework .

The source-coefficient calculations show why this duality is not only formal. The continuous source determines the scalar coefficient of the induced discrete flow. For Gaussian sources, this coefficient is delayed as vocabulary size grows. For bounded uniform sources, for fixed vocabulary size, the coefficient is controlled by support width. For centered negative-exponential sources, the coefficient has a simple vocabulary-independent form. These formulas suggest that source design should be treated as a structural part of continuous categorical flow modeling, not as a neutral implementation choice.

\newpage
\appendix
\section{Duality Proofs}\label{A}

\subsection{Proof of Lemma \ref{lem:argmax-measurable}}

\begin{proof}
Fix $j \in [V]$. Define
\begin{equation}
A(j) := \{u \in \mathbb{R}^V : \argmax(u) = j\}.
\end{equation}
By the definition of $\argmax$ with ties broken by the smallest index,
\begin{equation}
A(j)
= \left(\bigcap_{k<j}\{u\in\mathbb R^V: u_j>u_k\}\right)
\cap
\left(\bigcap_{k>j}\{u\in\mathbb R^V: u_j\ge u_k\}\right).
\end{equation}
This is a finite intersection of open or closed half-spaces in $\mathbb{R}^V$, and hence $A(j)$ is Borel.

For a sequence $x = (x^1,\dots,x^L) \in [V]^L$, its preimage under $\argmax$ is
\begin{equation}
\argmax^{-1}(\{x\})
= \prod_{i=1}^L A(x^i) 
\subset (\mathbb{R}^V)^L \cong \mathbb{R}^{LV}.
\end{equation}
A finite product of Borel sets is Borel in the product space, so $\argmax^{-1}(\{x\})$ is Borel for every $x \in [V]^L$. Therefore $\argmax$ is measurable as a map from $(\mathbb{R}^{LV}, \mathcal{B}(\mathbb{R}^{LV}))$ to the discrete space $([V]^L, \mathcal{P}([V]^L))$.
\end{proof}

\subsection{Proof of Lemma \ref{lem:uniform-source}}

\begin{proof}
Fix a position $i$. Since ties have probability zero, $\argmax(Z_0^i)$ is almost surely unique. For any two indices $j,k\in[V]$, choose a coordinate permutation $\sigma$ sending $j$ to $k$. We use the convention $(\sigma u)_r=u_{\sigma^{-1}(r)}$, so $\argmax(\sigma u)=\sigma(\argmax u)$ whenever the argmax is unique. By coordinate-permutation invariance, $Z_0^i$ and $\sigma Z_0^i$ have the same law. Therefore
\begin{equation}
\mathbb P(\argmax(Z_0^i)=j)
=\mathbb P(\argmax(\sigma Z_0^i)=k)
=\mathbb P(\argmax(Z_0^i)=k).
\end{equation}
All tokens have equal probability, and the probabilities sum to one, so each is $1/V$.

Because $\tilde p_0=\prod_{i=1}^L\tilde p_0^i$ is a product measure and the argmax is applied independently at each position, the coordinates of $X_0=\argmax(Z_0)$ are independent uniform random variables on $[V]$. Hence $p_0$ is the uniform distribution on $[V]^L$.
\end{proof}

\subsection{Proof of Lemma \ref{lem:factorization}}

\begin{proof}
By construction, $Z_0=(Z_0^1,\dots,Z_0^L)$ has independent components and $X_0^i=\argmax(Z_0^i)$ depends only on $Z_0^i$. Hence, conditional on $X_0=x_0$, the random variables $(Z_0^i)_{i=1}^L$ remain independent, with $Z_0^i$ distributed as $\tilde p_0^i$ conditioned on $\argmax(Z_0^i)=x_0^i$.

Under the lifted coupling construction, $X_1$ is sampled from the discrete kernel $\pi(\cdot\mid X_0)$; hence $Z_0\perp X_1\mid X_0$. Therefore
\[
\mathcal L(Z_0\mid X_0=x_0,X_1=x_1)=\mathcal L(Z_0\mid X_0=x_0),
\]
so conditioning additionally on $X_1=x_1$ does not introduce dependence among the variables $(Z_0^i)_{i=1}^L$. Moreover, $Z_1^i=\operatorname{onehot}(x_1^i)$ is deterministic given $X_1=x_1$.

Thus, conditional on $(X_0,X_1)=(x_0,x_1)$, the pairs $(Z_0^i,Z_1^i)$ are independent across positions. Since
\begin{equation}
Z_t^i=(1-\tilde k_t)Z_0^i+\tilde k_t Z_1^i
\end{equation}
is a measurable function of $(Z_0^i,Z_1^i)$, the variables $(Z_t^i)_{i=1}^L$ are conditionally independent given $(X_0,X_1)$. Finally, $X_t^i=\argmax(Z_t^i)$ depends only on $Z_t^i$, so $(X_t^i)_{i=1}^L$ are conditionally independent given $(X_0,X_1)$, which gives the claimed factorization.
\end{proof}

We now describe the form of each per-position conditional.

\subsection{Proof of Theorem \ref{thm:discrete-flow}}

\begin{proof}
By the lifted coupling construction, we sample $Z_0\sim\tilde p_0$, set $X_0=\argmax(Z_0)$, sample $X_1\sim\pi(\cdot\mid X_0)$, set $Z_1=\operatorname{onehot}(X_1)$, and form the convex interpolant
\[
Z_t=(1-\tilde k_t)Z_0+\tilde k_t Z_1.
\]

For a realization $(z_0,z_1)$ of $(Z_0,Z_1)$, write $z_t=(1-\tilde k_t)z_0+\tilde k_t z_1$. The cases $\tilde k_t=0$ and $\tilde k_t=1$ are the source and target endpoints and are immediate. In particular, if the lifted kernel has $X_1\sim q$, then $Z_1=\operatorname{onehot}(X_1)$ implies $p_1=q$ after argmax projection. We therefore consider $0<\tilde k_t<1$ and write $c_t=\tilde k_t/(1-\tilde k_t)$. For convex interpolants in continuous flows we have
\[
z_t=(1-\tilde k_t)z_0+\tilde k_t z_1,
\qquad\text{hence}\qquad
z_t^i=(1-\tilde k_t)z_0^i+\tilde k_t z_1^i.
\]
Write $j_0:=x_0^i$ and $j_1:=x_1^i$ for the maximizing indices. If $j_0=j_1$, then the boosted coordinate is already the source argmax, so the argmax remains $j_0$ for all $t$ and the conditional law is simply $\delta_{x_0^i}=\delta_{x_1^i}$. Suppose therefore that $j_0\neq j_1$.

Because $z_1^i$ is one-hot, the interpolated vector has the form
\[
z_t^i[j]=(1-\tilde k_t)z_0^i[j]\quad (j\neq j_1),
\qquad
z_t^i[j_1]=(1-\tilde k_t)z_0^i[j_1]+\tilde k_t.
\]
Among coordinates $j\neq j_1$, the maximum is attained at $j_0$, because $\argmax(z_0^i)=j_0$. Thus $\argmax(z_t^i)$ can only be either $j_0$ or $j_1$. More explicitly, on the event $\argmax(z_0^i)=j_0$ we have $\max_{j\neq j_1} z_0^i[j]=z_0^i[j_0]$, so, up to the equality event where tie-breaking may matter,
\[
\argmax(z_t^i)=j_1
\quad\Longleftrightarrow\quad
(1-\tilde k_t)z_0^i[j_1]+\tilde k_t \ >\ (1-\tilde k_t)z_0^i[j_0]
\quad\Longleftrightarrow\quad
z_0^i[j_0]-z_0^i[j_1] \ <\ c_t.
\]
The tie event at time $t$ is exactly $\{z_0^i[j_0]-z_0^i[j_1]=c_t\}$, which has probability zero by the pairwise-gap non-atomicity assumption. Thus using $\leq$ instead of $<$ gives the same transition probability. We use the convention that $c_t=+\infty$ when $\tilde k_t=1$.

If $x_1^i=x_0^i$, then $\delta_{x_0^i}=\delta_{x_1^i}$ and the conditional law is simply the same point mass regardless of the value assigned to $k_t^i$. Hence the mixture representation is degenerate in the equal-token case, and the coefficient only needs to be defined for $x_1^i\neq x_0^i$.

Assume therefore that $x_1^i\neq x_0^i$. At a time $t$ there is a probability $k_t^i(x_0^i,x_1^i)$ that $\argmax(z_t^i)=x_1^i$ and a probability $1-k_t^i(x_0^i,x_1^i)$ that $\argmax(z_t^i)=x_0^i$, namely
\[
p_{t|0,1}^i(x_t^i \mid x_0^i,x_1^i)
=(1-k_t^i(x_0^i,x_1^i))\,\delta_{x_0^i}(x_t^i)+k_t^i(x_0^i,x_1^i)\,\delta_{x_1^i}(x_t^i).
\]
We now show that this coefficient does not depend on the particular distinct pair $(x_0^i,x_1^i)$, but only on $\tilde k_t$ and $\tilde p_0^i$. For $x_1^i\neq x_0^i$, define
\begin{equation}
\label{eq:kt-def-correct}
k_t^i(x_0^i,x_1^i)
=
\frac{
\tilde p_0^i\!\Big(\Big\{z_0^i \ \Big|\ \argmax(z_0^i)=x_0^i,\ z_0^i[x_0^i]-z_0^i[x_1^i]\le c_t\Big\}\Big)
}{
\tilde p_0^i\!\Big(\big\{z_0^i \ \big|\ \argmax(z_0^i)=x_0^i\big\}\Big)
}.
\end{equation}
For the equal-token case, the two delta masses coincide, so the value of the coefficient is immaterial; we use the same value as in the distinct-token case for notational consistency.

To see why $k_t^i(x_0^i,x_1^i)$ does not depend on the particular distinct values of $x_0^i$ and $x_1^i$, use a permutation argument. Let $(a,b)$ and $(a',b')$ be any two ordered pairs with $a\neq b$ and $a'\neq b'$. Choose a coordinate permutation $\sigma$ with $\sigma(a)=a'$ and $\sigma(b)=b'$, again acting by $(\sigma u)_r=u_{\sigma^{-1}(r)}$. Since $Z^i\overset{d}=\sigma Z^i$ under the coordinate-permutation invariance assumption, the denominator events
\[
\{\argmax(Z^i)=a\}
\quad\text{and}\quad
\{\argmax(\sigma Z^i)=a'\}
\]
have the same probability. The same permutation maps the numerator event
\[
\{\argmax(Z^i)=a,\ Z^i_a-Z^i_b\le c_t\}
\]
onto
\[
\{\argmax(\sigma Z^i)=a',\ (\sigma Z^i)_{a'}-(\sigma Z^i)_{b'}\le c_t\}.
\]
Thus the numerator and denominator in \eqref{eq:kt-def-correct} are invariant under the map, and the ratio is identical for all distinct token pairs. Therefore $k_t^i(x_0^i,x_1^i)$ is constant across distinct $(x_0^i,x_1^i)$; denote this constant by $k_t^i\in[0,1]$. We conclude
\[
p_{t|0,1}^i(x_t^i \mid x_0^i,x_1^i)
=(1-k_t^i)\,\delta_{x_0^i}(x_t^i) + k_t^i\,\delta_{x_1^i}(x_t^i).
\]
with $k_t^i$ depending only on $\tilde k_t$ and the source $\tilde p_0^i$, not on $(x_0^i,x_1^i)$. If all per-position source laws are identical, then all $k_t^i$ are equal; in that case we write the common coefficient as $k_t$.

% \paragraph{Gaussian special case.}
% In the case of the source being i.i.d.\ standard normal entries, write $\Phi$ and $\varphi$ for the standard normal CDF and pdf, and take $\tilde k_t=t$. Then
% \[
% \tilde p_0^i\big(\{z_0^i\mid \argmax(z_0^i)=x_0^i\}\big)=\frac{1}{V}.
% \]
% Moreover, for $x_1^i\neq x_0^i$ and $c_t:=\frac{t}{1-t}$, the numerator in \eqref{eq:kt-def-correct} can be written as
% \[
% \tilde p_0^i\!\Big(\Big\{z_0^i \ \Big|\ \argmax(z_0^i)=x_0^i,\ z_0^i[x_0^i]-z_0^i[x_1^i]\le c_t\Big\}\Big)
% =
% \int_{-\infty}^{\infty} \varphi(m)\,\Phi(m)^{V-2}\,\big(\Phi(m)-\Phi(m-c_t)\big)\,dm,
% \]
% and therefore the (token-independent) mixing weight admits the explicit integral representation
% \[
% k_t
% =
% V\int_{-\infty}^{\infty} \varphi(m)\,\Phi(m)^{V-2}\,\big(\Phi(m)-\Phi(m-c_t)\big)\,dm,
% \qquad c_t=\frac{t}{1-t}.
% \]
\end{proof}

\section{Source coefficient derivations}\label{app:coefficient_derivations}

Theorem~\ref{thm:discrete-flow} gives the conditional discrete transition
\begin{equation}
    p_{t|0,1}^i(x_t^i\mid x_0^i,x_1^i)
    = (1-k_t)\delta_{x_0^i}(x_t^i)+k_t\delta_{x_1^i}(x_t^i),
\end{equation}
for $x_0^i\neq x_1^i$. This appendix derives $k_t$ for the source families used in Section~\ref{subsec:source_coefficients}. The plots use the linear interpolant $\tilde k_t=t$, while the derivations below are written in terms of the general ratio $c_t=\tilde k_t/(1-\tilde k_t)$. Throughout, fix a position $i$, write $V\ge 2$ for the vocabulary size, and let
\begin{equation}
    c_t = \frac{\tilde{k}_t}{1-\tilde{k}_t}.
\end{equation}
Condition on the source argmax being coordinate $a$ and the target being coordinate $b\neq a$. The target wins at time $t$ precisely when
\begin{equation}
    Z_a-Z_b \leq c_t
\end{equation}
under the additional conditioning that $Z_a$ is the maximum coordinate of the source.

\subsection{Gaussian source}

Let $Z_1,\dots,Z_V$ be i.i.d. standard normal. Consider the joint event that coordinate $a$ has value $m$ and is the argmax. Its unnormalized density is
\begin{equation}
    \varphi(m)\Phi(m)^{V-1},
\end{equation}
which integrates to $1/V$. Given $Z_a=m$ and the event that $a$ is the maximum, the remaining coordinates are i.i.d. normal variables truncated above by $m$. The event that the target coordinate $b$ is within $c_t$ of the maximum is $Z_b\geq m-c_t$. Therefore the unnormalized mass of source points for which the source argmax is $a$ and the target coordinate wins by time $t$ is
\begin{equation}
    \int_{-\infty}^{\infty}\varphi(m)\Phi(m)^{V-2}\bigl(\Phi(m)-\Phi(m-c_t)\bigr)\,dm.
\end{equation}
The conditioning event $\argmax(Z)=a$ has probability $1/V$. Dividing by $1/V$ gives
\begin{equation}
    k_t^{\mathrm{Gauss}}
    = V\int_{-\infty}^{\infty}\varphi(m)\Phi(m)^{V-2}\bigl(\Phi(m)-\Phi(m-c_t)\bigr)\,dm.
\end{equation}
An equivalent form follows by fixing the target coordinate but not conditioning on the source argmax. Let
\begin{equation}
    \rho_t^{\mathrm{Gauss}}
    = \mathbb{P}(\text{fixed target coordinate wins before conditioning on the source argmax})
    = \mathbb{E}_{G\sim N(0,1)}\big[\Phi(G+c_t)^{V-1}\big].
\end{equation}
Since the source is uniform after argmax projection,
\begin{equation}
    \rho_t^{\mathrm{Gauss}}=\frac{1}{V}+\frac{V-1}{V}k_t^{\mathrm{Gauss}},
\end{equation}
and hence
\begin{equation}
    k_t^{\mathrm{Gauss}}=\frac{V\rho_t^{\mathrm{Gauss}}-1}{V-1}.
\end{equation}

\subsection{Bounded uniform source}

Let $\epsilon>0$ and let $Z_1,\dots,Z_V$ be i.i.d. uniform on an interval $[a,a+\epsilon]$. The location $a$ cancels in all coordinate comparisons, so only the width $\epsilon$ matters. Define
\begin{equation}
    r_t = \min\left\{\max\left\{\frac{c_t}{\epsilon},0\right\},1\right\}.
\end{equation}
First compute the unconditional target-win probability. Given the target coordinate value $Z_b=z$, all other coordinates must be at most $z+c_t$. For $0\leq c_t\leq \epsilon$,
\begin{equation}
    \rho_t^{\mathrm{Unif}}
    = \frac{1}{\epsilon}\int_a^{a+\epsilon-c_t}\left(\frac{z+c_t-a}{\epsilon}\right)^{V-1}dz
      +\frac{1}{\epsilon}\int_{a+\epsilon-c_t}^{a+\epsilon}1\,dz.
\end{equation}
Changing variables gives
\begin{equation}
    \rho_t^{\mathrm{Unif}}=r_t+\frac{1-r_t^V}{V}.
\end{equation}
Using $\rho_t=1/V+(V-1)k_t/V$ yields
\begin{equation}
    k_t^{\mathrm{Unif}}=\frac{Vr_t-r_t^V}{V-1}.
\end{equation}

\subsection{Centered negative-exponential source}

Let $E_1,\dots,E_V$ be i.i.d. exponential variables with rate $\beta$, and set
\begin{equation}
    Z_j = \frac{1}{\beta}-E_j.
\end{equation}
The shift makes the source centered but does not affect argmax comparisons. The target wins unconditionally if every non-target coordinate is at most $Z_b+c_t$. Let $Y=\beta^{-1}-Z_b=E_b$. Then $Y\sim\mathrm{Exp}(\beta)$. If $Y<c_t$, the threshold exceeds the upper endpoint of the support, so the target wins with probability one. If $Y\geq c_t$, then
\begin{equation}
    \mathbb{P}(Z_j\leq Z_b+c_t)=\exp(-\beta(Y-c_t)).
\end{equation}
Therefore
\begin{align}
    \rho_t^{\mathrm{NegExp}}
    &= \mathbb{P}(Y<c_t)
    +\mathbb{E}\left[\exp\bigl(-(V-1)\beta(Y-c_t)\bigr)\mathbf{1}_{\{Y\geq c_t\}}\right]\\
    &= 1-e^{-\beta c_t}+\int_{c_t}^{\infty}\beta e^{-\beta y}e^{-(V-1)\beta(y-c_t)}dy\\
    &= 1-\left(1-\frac{1}{V}\right)e^{-\beta c_t}.
\end{align}
Using again $\rho_t=1/V+(V-1)k_t/V$, we obtain
\begin{equation}
    k_t^{\mathrm{NegExp}}=1-e^{-\beta c_t}.
\end{equation}
This coefficient is independent of the vocabulary size $V$.

\bibliography{refs}
\bibliographystyle{tmlr}

\end{document}